\documentclass[letterpaper]{article} % DO NOT CHANGE THIS
\usepackage[preprint]{aaai2027}  % DO NOT CHANGE THIS
\usepackage[hyphens]{url}  % DO NOT CHANGE THIS
\usepackage{graphicx} % DO NOT CHANGE THIS
\usepackage{natbib}  % DO NOT CHANGE THIS AND DO NOT ADD ANY OPTIONS TO IT
\usepackage{caption} % DO NOT CHANGE THIS AND DO NOT ADD ANY OPTIONS TO IT
\usepackage{booktabs}
\usepackage{amsmath,amssymb,amsfonts}
\usepackage{amsthm}
\usepackage{multirow}
\usepackage{adjustbox}
\usepackage{subcaption}
\usepackage{algorithm}
\usepackage{algpseudocode}
\usepackage{placeins}
\usepackage{needspace}

\theoremstyle{plain}

\newtheorem{proposition}{Proposition}
\newtheorem{remark}{Remark}

\title{Deciding When to Switch: E-Processes for Adaptive Minimax Training for Generative Adversarial Nets}

\author{
Hyunjoo Kim,\textsuperscript{\rm 1}
Sicheng Wu,\textsuperscript{\rm 2}
Agastya Venkatraman,\textsuperscript{\rm 3}
Guang Lin,\textsuperscript{\rm 2}
Sehwan Kim\textsuperscript{\rm 1}
}

\affiliations{
\textsuperscript{\rm 1}Department of Statistics, Ewha Womans University\\
\textsuperscript{\rm 2}Department of Mathematics, Purdue University\\
\textsuperscript{\rm 3}Department of Computer Science, Purdue University\\
hjkim0320@ewha.ac.kr,
wu2404@purdue.edu,
venkatr9@purdue.edu,\\
guanglin@purdue.edu,
sehwankim@ewha.ac.kr
}

\begin{document}

\maketitle

\begin{abstract}
Modern data science increasingly gives rise to hypothesis-testing problems that are not naturally formulated in terms of parameters within prespecified statistical models. One important example is the dynamic evaluation of optimization algorithms, where decisions must be made during training about whether further updates remain beneficial or the algorithm should switch to a different phase. This issue is particularly relevant in stochastic min--max optimization. Generative adversarial networks (GANs) provide a canonical example, as their training requires repeated decisions about when to switch between discriminator and generator updates, yet existing methods typically rely on fixed update ratios or heuristic criteria. We formulate this switching problem as sequential hypothesis testing and develop an e-process-based adaptive training procedure. During discriminator updates, one e-process tests the null that the discriminator-induced separation between the empirical data distribution and the generator law remains below a target level. During generator updates, with the discriminator fixed, a second e-process tests the reverse null that this separation remains above a refresh level. Conditional on the observed training sample, we prove that fresh empirical indices and latent draws yield conditional e-values that can be accumulated into e-processes, providing anytime-valid Type I error control under adaptive model updates and data-dependent switching. Across multimodal synthetic distributions and image benchmark datasets, the proposed method matches or outperforms the best fixed-ratio baselines under several widely used GAN objectives.

\end{abstract}

\section{Introduction}
\label{sec:introduction}

Statistical hypothesis testing has been a basic tool for drawing conclusions from data, enabling principled decisions across a wide range of disciplines. Classical testing has predominantly focused on hypotheses concerning parameters or distributions within pre-specified statistical models. However, many recent applications raise testing problems that are not naturally expressed as hypotheses about parameters in a statistical model. For example, one may wish to test whether a given text was generated by a large language model rather than written by a human~\citep{Varshney2020LimitsOD,Li_2025,zhou2026detectingllmgeneratedtextperformance}. Another example arises in differential privacy, where privacy guarantees can be formulated in terms of the difficulty of distinguishing between outputs generated from neighboring datasets~\citep{kairouz2015composition,dong2022gdp,su2025statistical}.

Beyond these emerging questions, there is also a need for principled methods to guide decisions made during the training of modern machine-learning systems. One important setting is min--max optimization, which arises in generative adversarial networks (GANs), adversarially robust learning, and reinforcement learning, including zero-sum Markov games and minimax methods for off-policy evaluation~\citep{goodfellow2014gan,madry2018towards,littman1994markov,uehara2020minimax}. In these problems, two competing components are optimized against each other, and the progress of one component depends on the current state of the other.

Among these applications, generative adversarial networks provide a particularly important and challenging example. Although GANs have played a foundational role in the development of deep generative modeling, their training is notoriously unstable because the generator and discriminator evolve simultaneously and continually alter each other's optimization objectives. These coupled dynamics often make GAN training difficult, leading to problems such as vanishing gradients and mode collapse.

A large body of work has therefore sought to stabilize GAN training by redesigning the adversarial objective, including $f$-divergence, least-squares, Wasserstein, and hinge losses~\citep{goodfellow2014gan,nowozin2016fgan,mao2017least,arjovsky2017wgan,lim2017geometric}. Other approaches regularize or constrain the discriminator through techniques such as gradient penalties and spectral normalization, with theoretical analyses showing that suitable gradient regularization can improve the local convergence properties of GAN training~\citep{gulrajani2017improved,miyato2018spectral,mescheder2018training}.

A separate practical issue is how to allocate updates between the discriminator and generator. In most GAN algorithms, the update schedule is fixed before training begins. The original GAN experiments largely use a one-to-one schedule, whereas Wasserstein GAN (WGAN) and spectral normalization GAN (SNGAN) typically perform several critic or discriminator updates for each generator update. 
A few approaches have attempted to better balance the training dynamics. One approach assigns different learning rates to the two networks~\citep{heusel2017ttur}, while \citet{ouyang2021accelerated} selects which network to update based on changes in the observed losses. The first approach still follows a prespecified update schedule. The second was developed specifically for the WGAN setting, making its extension to other GAN objectives unclear. What is still missing is a general rule for deciding when to stop updating one network and switch to the other. We address this problem by developing a statistically principled, data-adaptive scheduling rule that applies to a broad class of GAN objectives

In this work, we formulate the update-scheduling problem in GAN training as a sequential hypothesis-testing problem. At each stage of training, the algorithm decides whether the discriminator or generator currently being updated has made sufficient progress or whether further updates are still needed. Unlike classical testing problems, the hypotheses are examined repeatedly along an optimization trajectory, and the decision to continue updating the current network or switch to the other depends on the evidence accumulated during training.

We use e-values and e-processes to determine whether the algorithm should continue updating the current network or switch to the other. An e-value is a nonnegative measure of evidence against the null hypothesis whose expectation under the null is at most one~\citep{vovk2021evalues}. An e-process extends this idea to sequential settings and supports anytime-valid testing under continuous monitoring and data-dependent stopping~\citep{ramdas2023gametheoretic,grunwald2024safe}. We construct e-processes to decide whether the current network should continue to be updated. Once the accumulated evidence supports ending the current discriminator or generator phase, the algorithm switches to the other network. The resulting procedure replaces a prespecified update schedule with a statistically principled, data-adaptive scheduling rule.

\section{Preliminary}
\label{sec:preliminary}

\paragraph{E-value}

Let $\mathcal{P}$ denote a class of distributions representing the null hypothesis. A nonnegative random variable $E\geq 0$ is called an \emph{e-variable} for $\mathcal{P}$ if
\[
\mathbb{E}_{P}[E] \leq 1
\qquad
\text{for every } P\in\mathcal{P},
\]
and it is called \emph{exact} if equality holds for every $P\in\mathcal{P}$. The realized value of an e-variable is referred to as an \emph{e-value}~\citep{vovk2021evalues}.

E-values admit a direct interpretation as measures of evidence against the null hypothesis. By Markov's inequality, for every $P\in\mathcal{P}$ and $\alpha\in(0,1]$, 
\[P\left(E\geq \frac{1}{\alpha}\right)
\leq \alpha.\]
Thus, an e-value exceeding $1/\alpha$ provides level-$\alpha$ evidence against the null hypothesis.

Validity alone does not ensure that an e-value provides meaningful evidence against the null. For example, the constant $E\equiv 1$ is exact for any null distribution, and any constant $E\equiv c$ with $0\leq c\leq 1$ is also valid. However, these e-values take the same value under both the null and alternative hypotheses and therefore provide no evidence for distinguishing between them. A useful e-value should instead tend to take large values under the alternative while maintaining the required expectation bound under the null.

As an example, consider the simple hypothesis test
\begin{equation}\label{HP:simple_simple}
    H_0: X\sim P
\qquad\text{versus}\qquad
H_1: X\sim Q,
\end{equation}
a canonical e-variable is the likelihood ratio $E(X)=\frac{dQ(X)}{dP(X)},$ assuming that $Q$ is absolutely continuous with respect to $P$. Indeed, the likelihood ratio is an exact e-variable, and it becomes large when the observation is more likely under $Q$ than under $P$. E-values can be merged in several ways. Suppose that \(E_1,\ldots,E_m\) are e-variables for the same null model \(\mathcal P\). For any nonnegative deterministic weights \(w_1,\ldots,w_m\) satisfying \(\sum_{j=1}^m w_j=1\), the weighted average 
\[
\overline E=
\sum_{j=1}^m w_jE_j
\]
is again an e-variable. If \(E_1,\ldots,E_m\) are further independent under every \(P\in\mathcal P\), then
\[
E_{\mathrm{prod}}=
\prod_{j=1}^m E_j
\]
is an e-variable.

\paragraph{E-process}

An e-process extends the concept of an e-value to a sequential setting in which observations and evidence are revealed over time. Let $(\mathcal{F}_t)_{t\geq 0}$ denote the filtration representing the information available during the sequential experiment. A nonnegative adapted process $(M_t)_{t\geq 0}$ with $M_0=1$ is called an \emph{e-process} for $\mathcal{P}$ if
\[
\mathbb{E}_{P}[M_{\tau}] \leq 1
\]
for every $P\in\mathcal{P}$ and every stopping time $\tau$.

A common way to construct an e-process is through sequential multiplication. Suppose that, at each time $t$, a nonnegative random variable $E_t$ is chosen using the past information $\mathcal{F}_{t-1}$ and satisfies
\[
\mathbb{E}_{P}
\left[
E_t\mid\mathcal{F}_{t-1}
\right]
\leq 1
\qquad
\text{for every } P\in\mathcal{P}.
\]
Such an $E_t$ is called a conditional e-variable. The product process
\[
M_t=\prod_{s=1}^{t}E_s,
\qquad M_0=1,
\]
is then a nonnegative supermartingale under every $P\in\mathcal{P}$ and therefore forms an e-process. 

A key property of an e-process is its anytime-validity. By Ville's inequality, for every \(\alpha\in(0,1]\),
\[
P
\left(
\sup_{t\geq 1}M_t
\geq \frac{1}{\alpha}
\right)
\leq \alpha
\qquad
\text{for every } P\in\mathcal{P}.
\]
Consequently, the stopping rule $\tau_\alpha
=
\inf\left\{
t\geq 1:
M_t\geq \frac{1}{\alpha}
\right\}$ yields a level-\(\alpha\) sequential test that rejects the null hypothesis $P\in \mathcal P$. For a comprehensive review of e-values and e-processes, see \citet{ramdas2025evalues}.

\paragraph{Conformal e-prediction}

As illustrated by the density-ratio example, the likelihood ratio provides a canonical construction of an informative e-value. In many machine-learning applications, however, neither hypothesis is specified by an explicit probabilistic model, so the corresponding likelihood ratio cannot be evaluated directly. One may instead estimate a score that approximates the density ratio using flexible machine-learning methods, such as deep neural networks~\citep{tomaselli2025RSBI}. Approximation error, optimization error, regularization, and finite-sample generalization error may cause its expectation under the null to exceed one.

Conformal e-prediction offers a way to calibrate a fitted score into a valid e-value under exchangeability. Consider the testing problem in \eqref{HP:simple_simple}. Let \(X_1,\ldots,X_m\) be calibration observations drawn from the null distribution \(P\), and let \(X_{m+1}\) be a new observation to be tested. Under \(H_0\), the augmented sample \(X_1,\ldots,X_m,X_{m+1}\) is exchangeable. Let \(A:\mathcal X\to[0,\infty)\) be a nonnegative evidence score learned from an independent training sample, with larger values indicating stronger evidence in favor of \(Q\) over \(P\). The score \(A\) may approximate the density ratio \(dQ/dP\), but it need not be a density ratio.

The conformal e-value is then defined as
\[
E_{\mathrm{conf}}(X_{m+1})
:=
\frac{(m+1)A(X_{m+1})}
{\sum_{j=1}^{m+1}A(X_j)}.
\]
If \(\sum_{j=1}^{m+1}A(X_j)=0\), we define $E_{\mathrm{conf}}(X_{m+1}):=1.$ Under \(H_0\), exchangeability implies that the \(m+1\) terms on the left-hand side have the same expectation. Consequently,
\[
\mathbb E_{P}
\left[
E_{\mathrm{conf}}(X_{m+1})
\right]
=1,
\]
so \(E_{\mathrm{conf}}(X_{m+1})\) is a valid e-value. Thus, conformal normalization restores finite-sample validity even when the underlying evidence score is only an approximate or machine-learned density-ratio surrogate.

\begin{remark}
    
 Although any nonnegative score yields a valid conformal e-value, its power depends strongly on the choice of score. Suppose that the calibration observations are drawn from the null distribution \(P\), while the test observation is drawn from the alternative distribution \(Q\). For a large calibration sample,
\[
E_{\mathrm{conf}}(X_{m+1})
\approx
\frac{A(X_{m+1})}{\mathbb E_P[A(X)]}. 
\]
Then, the expected log-evidence under \(Q\) is approximately
\[
\mathbb E_Q[\log E_{\mathrm{conf}}(X)]
\approx
\mathbb E_Q[\log A(X)]
-
\log \mathbb E_P[A(X)].
\]
Provided that \(Q\ll P\), this criterion satisfies
\[
\mathbb E_Q[\log A(X)]
-
\log \mathbb E_P[A(X)]
\leq
D_{\mathrm{KL}}(Q\|P),
\]
with equality when $A(x)\propto\frac{dQ(x)}{dP(x)}.$ Thus, conformal calibration guarantees finite-sample validity under exchangeability, while density-ratio estimation remains a natural principle for constructing powerful evidence scores.

\end{remark}

\section{Representations of Distributional discrepancies}

Let $P$ and $Q$ be two probability distributions on $\mathcal X$. Many adversarial learning methods compare $P$ and $Q$ by optimizing over a class of score functions that distinguish between the two distributions. The resulting optimization problem has a variational representation of  distributional discrepancy. In general, we consider a representation of the form
\begin{equation}
\label{eq:general_variational}
\mathcal D(P,Q)
=
\sup_{s\in\mathcal S}
\left\{
    \mathbb E_{P}\!\left[\phi\{s(X)\}\right]
    -
    \mathbb E_{Q}\!\left[\psi\{s(Y)\}\right]
\right\},
\end{equation}
where \(s:\mathcal X\to\mathbb R\) is a score function, \(\mathcal S\) is a prescribed class of admissible score functions, and $\phi$ and $\psi$ determine how the score is evaluated under $P$ and $Q$, respectively.

Two important classes of distributional discrepancies that admit such variational representations are \(f\)-divergences and integral probability metrics.

\paragraph{\(f\)-divergences}
Let \(f:(0,\infty)\to\mathbb R\) be a convex function satisfying \(f(1)=0\). Assuming \(P\ll Q\), the corresponding \(f\)-divergence is defined as
\begin{equation*}
\label{eq:fdiv_primal}
D_f(P\|Q)
=
\mathbb E_{Y\sim Q}
\left[
    f\left(\frac{dP}{dQ}(Y)\right)
\right].
\end{equation*}
Let \(f^*\) denote the Fenchel conjugate of \(f\), defined by $f^*(t)
=
\sup_{u>0}\{ut-f(u)\}.$ Under suitable regularity and integrability conditions,
\(D_f(P\|Q)\) admits the variational representation
\begin{equation}
\label{eq:fdiv_variational}
D_f(P\|Q)
=
\sup_{s\in\mathcal S_f}
\left\{
    \mathbb E_{ P}[s(X)]
    -
    \mathbb E_{ Q}\!\left[f^*\{s(Y)\}\right]
\right\},
\end{equation}
where \(\mathcal S_f\) is a class of measurable score functions whose ranges lie in \(\operatorname{dom}(f^*)\).

\paragraph{Integral probability metrics}

An integral probability metric (IPM) associated with a function class \(\mathcal S\) is defined by
\begin{equation}
\label{eq:ipm}
D_{\mathcal S}(P,Q)
=
\sup_{s\in\mathcal S}
\left|
    \mathbb E_{P}[s(X)]
    -
    \mathbb E_{ Q}[s(Y)]
\right|.
\end{equation}
If \(\mathcal S\) is symmetric, in the sense that \(s\in\mathcal S\) implies \(-s\in\mathcal S\), the absolute value in \eqref{eq:ipm} can be omitted. In contrast to the \(f\)-divergence
representation, the same score function appears in both expectations, while the choice of \(\mathcal S\) determines the resulting discrepancy. For example, taking \(\mathcal S_{\mathrm{TV}}
=
\{s:\|s\|_\infty\leq 1/2\}\) yields the total variation distance, while taking \(\mathcal S\) to be the class of \(1\)-Lipschitz functions yields the Wasserstein-1 distance.

In practice, \(P\) and \(Q\) are unknown and observed only through samples. Replacing the population expectations in \eqref{eq:general_variational} with empirical averages and optimizing yields an estimated score \(\widehat{s}\). For \(f\)-divergences, the population-optimal score satisfies $
s^*(x)=f'\left(\frac{dP}{dQ}(x)\right),$ under suitable regularity conditions, so \(\widehat{s}\) estimates a known transformation of the density ratio. In contrast, IPM-optimal scores generally lack such a representation and are identifiable only up to an additive constant. Nevertheless, in both settings, the fitted score typically assigns larger values to observations that are more likely under \(P\) than under \(Q\), and smaller values to those more likely under \(Q\).

Based on these observations about the fitted score, the next section develops e-processes for adaptively determining the discriminator and generator update schedules in GAN training.

\section{E-Process-Guided Adaptive GAN Training}

Let \(P\) denote the unknown data distribution on \(\mathcal X\), from which real observations \(X\sim P\) are drawn. Let \(D_{\theta_d}:\mathcal X\to\mathbb R\) be a neural network parameterized by \(\theta_d\), and let \(G_{\theta_g}:\mathcal U\to\mathcal X\) be a generator parameterized by \(\theta_g\), where \(\mathcal U\) denotes an auxiliary noise space. Given a reference distribution \(P_U\) on \(\mathcal U\), the generator induces the distribution $Q_{\theta_g}
=
(G_{\theta_g})_{\#}P_U$ of \(G_{\theta_g}(U)\), where \(U\sim P_U\).

To make \(Q_{\theta_g}\) close to the data distribution \(P\), we consider an adversarial game between a discriminator and a generator, with the following objectives:
\[
\begin{split}
\min_{\theta_d}
\mathcal L_D(\theta_d;\theta_g)
&=
\mathbb E_{P}
\left[
\phi(D_{\theta_d}(X))
\right]
+
\mathbb E_{Q_{\theta_g}}
\left[
\psi_1(D_{\theta_d}(Y))
\right],\\
\min_{\theta_g}
\mathcal L_G(\theta_g;\theta_d)
&=
\mathbb E_{ Q_{\theta_g}}
\left[
\psi_2(D_{\theta_d}(Y))
\right].
\end{split}
\]
Here, the functions \(\phi\), \(\psi_1\), and \(\psi_2\) are determined by the chosen distributional discrepancy. Under its variational representation, the discriminator is trained to distinguish \(P\) from \(Q_{\theta_g}\), while the generator is updated to make the two distributions less distinguishable. The pair \((D_{\theta_d},G_{\theta_g})\) is typically trained through alternating stochastic-gradient updates according to a prespecified schedule. We instead use sequential evidence to adaptively determine the number of discriminator and generator updates.

\subsection{Adaptive discriminator updates}
\label{sec:gen_framework}

The discriminator should be strong enough to provide useful signals to the generator, but it should not be trained until it becomes saturated. If the discriminator separates real and generated samples too well, then its output may become almost deterministic, leading to weak or unstable gradients for the generator. To avoid this behavior, we use an e-process to determine whether the discriminator has already learned sufficient separation between the real and generated samples. Once the accumulated evidence is sufficiently large, we stop updating the discriminator and switch to updating the generator.

We first describe the population-level motivation. Let $P$ denote the target real distribution and let $Q$ denote the current generator distribution. We consider the composite null hypothesis
\begin{equation}\label{HP:Discriminator}
        H_0^D(a_D):\quad \mathrm{TV}(P,Q)\le a_D,
\end{equation}
where $TV(\cdot , \cdot)$ is the total variation between two distributions. The null states that the current generator distribution is sufficiently close to the target distribution.

And we use the discriminator as a witness function that provides evidence against the null. The following proposition presents two constructions of an e-variable based on an arbitrary discriminator \(D:\mathcal X\to\mathbb R\).

\begin{proposition}
\label{prop:discriminator-evalue}
Suppose that \(X\sim P\) and \(Y\sim Q\) are independently drawn and that \(D\) is fixed before observing \((X,Y)\). Under \(H_0^D(a_D)\), each of the following is an e-variable:
\begin{enumerate}
    \item[(i)] The separable-score e-variable
    \[
    E_D^{\mathrm{sep}}(X,Y)
    =
    \frac{
        4\sigma\{D(X)\}\sigma\{-D(Y)\}
    }{
        (1+a_D)^2
    }.
    \]

    \item[(ii)] The difference-score e-variable
    \[
    E_D^{\mathrm{diff}}(X,Y)
    =
    \frac{
        2\sigma\{D(X)-D(Y)\}
    }{
        1+2a_D-a_D^2
    }.
    \]
    Here $\sigma(\cdot)$ denotes sigmoid function. 
\end{enumerate}
\end{proposition}
\begin{proof}[Sketch of proof] For separable case, using the fact
\[
\mathbb E_{P}[\sigma(D(X))]
-
\mathbb E_{Q}[\sigma(D(Y))]
\le
\operatorname{TV}(P,Q)
\le a_D,
\]
the independence of \(X\) and \(Y\), and the inequality
\(4uv\le (u+v)^2\), we obtain
\begin{align*}
&\mathbb E_{P\times Q}
\left[
4\sigma\{D(X)\}\sigma\{-D(Y)\}
\right]
\le (1+a_D)^2.
\end{align*}
The difference case follows by an analogous argument. Further details are provided in the Supplementary Material below.
\end{proof}

Next, to characterize discriminators that yield rapid evidence accumulation, consider the alternative
\begin{equation}
\label{HP:Discriminator-alternative}
H_1^D(b_D):
\quad
\operatorname{TV}(P,Q)\ge b_D,
\end{equation}
where $0\le a_D<b_D\le 1.$ For fixed \(P\) and \(Q\), we measure power by the expected log-growth
\[
\Gamma_{P,Q}(D)
:=
\mathbb E_{P\times Q}\!\left[\log E_D(X,Y)\right],
\]
where a larger value indicates faster evidence accumulation.

\begin{proposition}
\label{prop:oracle-epower}
Suppose that \(P\) and \(Q\) admit strictly positive densities \(p\) and \(q\), respectively, with respect to a common dominating measure. For \(k\in\{\mathrm{sep},\mathrm{diff}\}\), the logarithmic e-power \(\Gamma_{P,Q}^{k}(D)\) is maximized by \(D_k^*(x)=\log\{p(x)/q(x)\}+c_k\), where
\(c_{\mathrm{sep}}=0\) and \(c_{\mathrm{diff}}\in\mathbb R\). Moreover,
\[
\begin{split}
    \sup_D\Gamma_{P,Q}^{\mathrm{sep}}(D)
 &=2\operatorname{JS}(P,Q)-2\log(1+a_D),  \\
 \sup_D\Gamma_{P,Q}^{\mathrm{diff}}(D)
&=\operatorname{JS}(P\times Q,Q\times P)
-\log(1+2a_D-a_D^2),
\end{split}
\]
Consequently, the corresponding oracle e-process has uniformly positive
logarithmic growth over \(H_1^D(b_D)\) whenever
\[
h(b_D)>
\begin{cases}
\log(1+a_D), & k=\mathrm{sep},\\
\log(1+2a_D-a_D^2), & k=\mathrm{diff},
\end{cases}
\]
where $h(t)=\frac12\{(1+t)\log(1+t)+(1-t)\log(1-t)\}.$

\end{proposition}

This result suggests that, for high e-power, the discriminator should approximate the log-density ratio \(\log(p/q)\). This is naturally connected to variational \(f\)-divergence estimation, for which the population-optimal critic is typically a known transformation of \(p/q\). In contrast, an IPM-optimal critic need not approximate the log-density ratio and therefore is not guaranteed to maximize e-power. Nevertheless, the difference-score construction is invariant under the shift \(D\mapsto D+c\), making it compatible with the additive non-identifiability of IPM critics.

We now describe the construction of the e-process for adaptively determining the number of discriminator updates. Let \(D_{\theta_d^{t,\ell}}\) denote the discriminator at the \(\ell\)-th inner iteration of outer iteration \(t\), and let \(G_{\theta_g^t}\) denote the generator, which is held fixed throughout the discriminator-update phase of outer iteration \(t\). Write \(Q_{\theta_g^t}\) for the distribution induced by \(G_{\theta_g^t}\). At outer iteration \(t\) and inner iteration \(\ell\), we draw a fresh evaluation mini-batch, for \(b=1,\ldots,B\),
\[
X_{t,\ell,b}\stackrel{\mathrm{iid}}{\sim}P,
\quad
U_{t,\ell,b}\stackrel{\mathrm{iid}}{\sim}P_U,
\quad
Y_{t,\ell,b}
=
G_{\theta_g^t}(U_{t,\ell,b})
\stackrel{\mathrm{iid}}{\sim}Q_{\theta_g^t}.
\]
For a chosen construction \(k_D\in\{\mathrm{sep},\mathrm{diff}\}\), we then compute
\[
E_{t,\ell,b}^D
=
E_{D_{\theta_d^{t,\ell}}}^{k_D}
\left(X_{t,\ell,b},Y_{t,\ell,b}\right).
\]
Here, \(E_D^{\mathrm{sep}}\) and \(E_D^{\mathrm{diff}}\) denote the two e-variables introduced in Proposition~\ref{prop:discriminator-evalue}. The mini-batch e-value is defined as
\[
E_{t,\ell}^{D,\mathrm{mb}}
=
\prod_{b=1}^B
\left(
\frac{1}{2}
+
\frac{1}{2}E_{t,\ell,b}^D
\right),
\]
and the mini-batch e-values are accumulated within outer iteration \(t\) as
\[
\mathcal E_{t,0}^D=1,
\qquad
\mathcal E_{t,\ell}^D
=
\prod_{j=1}^{\ell}
\left\{
1-\rho_D+\rho_D E_{t,j}^{D,\mathrm{mb}}
\right\}.
\]
Then \(\{\mathcal E_{t,\ell}^D\}_{\ell\geq0}\) is a nonnegative
supermartingale under \(H_{0,t}^D(a_D)\). We define the discriminator
stopping time as
\[
\tau_t^D
=
\inf\left\{
\ell\geq1:
\mathcal E_{t,\ell}^D
\geq\frac{1}{\alpha_D}
\right\}
\wedge L_D,
\]
where \(L_D\) is the maximum number of discriminator updates.

At \(\tau_t^D\), we terminate the discriminator-update phase and switch to the generator-update phase. Thus, the number of discriminator updates is determined adaptively, when sufficient evidence of separation is accumulated. 

\begin{remark}
Population-level validity requires fresh evaluation observations that were not used to train the discriminator or generator. When such observations are unavailable, we condition on the training sample and regard $\widehat P_n=\frac{1}{n}\sum_{i=1}^n\delta_{X_i}$ as fixed. This reflects the usual practice in GAN training of treating the observed empirical distribution as a representative proxy for \(P\).

After \(D_{\theta_d^{t,\ell}}\) and \(G_{\theta_g^t}\) have been fixed, draw \(I\sim\operatorname{Unif}\{1,\ldots,n\}\) and \(U\sim P_U\) independently, and set \(X=X_I\) and \(Y=G_{\theta_g^t}(U)\). Under
\[
H_{0,n,t}^D(a_D):
\quad
\operatorname{TV}\!\left(\widehat P_n,Q_{\theta_g^t}\right)\le a_D,
\]
\(E_{D_{\theta_d^{t,\ell}}}^{k_D}(X,Y)\) is conditionally e-valid given
the training data and past history.

If \(Q_{\theta_g^t}\) is non-atomic, then \(\operatorname{TV}(\widehat P_n,Q_{\theta_g^t})=1\), and hence the null is false for \(a_D<1\). This does not affect validity under the stated empirical null, but its interpretation is relative to the empirical training distribution rather than the population distribution \(P\). We may instead consider a discrepancy induced by the current discriminator. Let define
\[
\begin{split}
    \operatorname{TV}_{D}^{\mathrm{sep}}(P,Q)
&=
\mathbb E_P[\sigma\{D(X)\}]
-
\mathbb E_Q[\sigma\{D(Y)\}] \\
\operatorname{TV}_{D}^{\mathrm{diff}}(P,Q)
&=
\mathbb E_{P\times Q}
\left[
2\sigma\{D(X)-D(Y)\}-1
\right]. 
\end{split}
\]
Accordingly, for \(k_D\in\{\mathrm{sep},\mathrm{diff}\}\), for the
operational null
\[
\widetilde H_{0,n,t}^{D,k_D}(a_D):
\quad
\operatorname{TV}_{D_{\theta_d^{t,\ell}}}^{k_D}
\left(
\widehat P_n,Q_{\theta_g^t}
\right)
\le a_D
\quad
\text{for all } \ell\ge 0.
\]
Under this null, the corresponding construction in
Proposition~\ref{prop:discriminator-evalue} remains conditionally
e-valid, provided that \(D_{\theta_d^{t,\ell}}\) is fixed before the evaluation pair is observed. Consequently, \(\{\mathcal E_{t,\ell}^D\}_{\ell\geq0}\) remains a nonnegative supermartingale under operational nulls. 
\end{remark}

\subsection{Adaptive generator updates}
\label{sec:gen_generator}

At the beginning of the generator-update phase at outer iteration \(t\), set $ \theta_g^{t,0}=\theta_g^t$,  $Q_{t,0}=Q_{\theta_g^{t,0}}$, and fix the discriminator at $\widetilde D_t = D_{\theta_d^{t,\tau_t^D}},$ where \(\widetilde D_t\) has been trained to distinguish \(P\) from \(Q_{t,0}\). At generator iteration \(\ell\), the generator parameter is updated from \(\theta_g^{t,\ell-1}\) to \(\theta_g^{t,\ell}\), inducing the distribution \[ Q_{t,\ell} = Q_{\theta_g^{t,\ell}}. \] As \(Q_{t,\ell}\) changes from its initial distribution \(Q_{t,0}\), the discriminatory power of the fixed score \(\widetilde D_t\) may gradually deteriorate. We naturally seek to detect whether $\operatorname{TV}(P,Q_{t,\ell})<b_G $ at some generator iteration \(\ell\). Accordingly, we consider the global null hypothesis 
\[ H_{\infty,t}^{G}(b_G): \quad \operatorname{TV}(P,Q_{t,\ell})\geq b_G \quad\text{for all }\ell\geq0. 
\]

Evaluating \(\operatorname{TV}(P,Q_{t,\ell})\) at each generator iteration would require optimizing over the class of admissible test functions. During the generator phase, however, \(\widetilde D_t\) is held fixed and no such re-optimization is performed. We therefore use \(\widetilde D_t\) as a fixed witness of the evolving discrepancy. Let \(\widetilde d_t=\sigma\circ\widetilde D_t\) and define
\[
\Delta_{t,\ell}(\widetilde D_t)
=
\mathbb E_P[\widetilde d_t(X)]
-
\mathbb E_{Q_{t,\ell}}[\widetilde d_t(Y)].
\]
We then consider the operational global null
\begin{equation}\label{HP:generator}
    \widetilde H_{\infty,t}^{G}(b_G):
\quad
\Delta_{t,\ell}(\widetilde D_t)\ge b_G
\quad\text{for all }\ell\ge0.
\end{equation}
Since $\Delta_{t,\ell}(\widetilde D_t) \le \operatorname{TV}(P,Q_{t,\ell}),$ the operational null \(\widetilde H_{\infty,t}^{G}(b_G)\) is sufficient
for \(H_{\infty,t}^{G}(b_G)\). Rejection of the operational null may occur either because \(\operatorname{TV}(P,Q_{t,\ell})<b_G\), or because the fixed discriminator \(\widetilde D_t\) no longer captures the discrepancy between \(P\) and \(Q_{t,\ell}\) effectively. In either case, the fixed discriminator no longer provides sufficient evidence that the two distributions remain separated by at least \(b_G\). This provides a principled criterion for terminating the generator-update phase and restarting the discriminator-update phase.

\begin{proposition}
\label{prop:generator-evalue}
For any generator iteration \(\ell\geq0\), suppose that \(X\sim P\) and \(Y\sim Q_{t,\ell}\) are independently drawn, and that \(\widetilde D_t\) is fixed before observing \((X,Y)\). Under the operational null hypothesis \eqref{HP:generator}, the following quantities are e-variables:
\[
\begin{aligned}
E_{\widetilde D_t}^{G,\mathrm{sep}}(X,Y)
&=
\frac{
4\sigma\{\widetilde D_t(Y)\}
\sigma\{-\widetilde D_t(X)\}
}{
(1-b_G)^2
},\\
E_{\widetilde D_t}^{G,\mathrm{diff}}(X,Y)
&=
\frac{
\sigma\{\widetilde D_t(Y)-\widetilde D_t(X)\}
}{
1-b_G
},
\end{aligned}
\]
where \(0\le b_G<1\).
\end{proposition}

We now describe the construction of the e-process for adaptively determining the number of generator updates. 

At outer iteration \(t\) and generator iteration \(\ell\), we draw a fresh evaluation mini-batch, for \(b=1,\ldots,B\),
\[
X_{t,\ell,b}^{G}\stackrel{\mathrm{iid}}{\sim}P,
\quad
U_{t,\ell,b}^{G}\stackrel{\mathrm{iid}}{\sim}P_U,
\quad
Y_{t,\ell,b}^{G}
=
G_{\theta_g^{t,\ell}}(U_{t,\ell,b}^{G})
\stackrel{\mathrm{iid}}{\sim}Q_{t,\ell}.
\]
For a chosen construction
\(k_G\in\{\mathrm{sep},\mathrm{diff}\}\), we compute
\[
E_{t,\ell,b}^{G}
=
E_{\widetilde D_t}^{G,k_G}
\left(X_{t,\ell,b}^{G},Y_{t,\ell,b}^{G}\right),
\]
where \(E_{\widetilde D_t}^{G,\mathrm{sep}}\) and
\(E_{\widetilde D_t}^{G,\mathrm{diff}}\) are given in
Proposition~\ref{prop:generator-evalue}. Define
\[
\begin{split}
    E_{t,\ell}^{G,\mathrm{mb}}
&=
\prod_{b=1}^{B}
\left(
\frac12+\frac12E_{t,\ell,b}^{G}
\right) \\
\mathcal E_{t,\ell}^{G}
&=
\prod_{j=1}^{\ell}
\left\{
1-\rho_G+\rho_GE_{t,j}^{G,\mathrm{mb}}
\right\}, \quad \mathcal E_{t,0}^{G}=1
\end{split}
\]
for $\rho_G\in[0,1].$ Then \(\{\mathcal E_{t,\ell}^{G}\}_{\ell\ge0}\) is a nonnegative
supermartingale under \(\widetilde H_{\infty,t}^{G}(b_G)\). Therefore, we use stopping time for generator updating phase as 
\[
\tau_t^G
=
\inf\left\{
\ell\ge1:
\mathcal E_{t,\ell}^{G}
\ge\frac1{\alpha_G}
\right\}
\wedge L_G,
\]
where \(L_G\) is the maximum number of generator updates. At \(\tau_t^G\), we terminate the generator phase and set $\theta_g^{t+1} = \theta_g^{t,\tau_t^G}$. Because this changes the generator distribution from \(Q_{\theta_g^t}\) to \(Q_{\theta_g^{t+1}}\), we then restart the discriminator phase, initialize $ \theta_d^{t+1,0} = \theta_d^{t,\tau_t^D},$ and reset its e-process by setting $\mathcal E_{t+1,0}^{D}=1.$ 

\begin{remark}
    We emphasize that the anytime-validity guarantee is phase-specific, because the e-process is reset after each switch. 
\end{remark}

\section{Experiments}
\label{sec:experiments}

\subsection{Mixture of 16 Gaussians}

We first consider a synthetic two-dimensional Gaussian mixture with 16 equally weighted modes arranged on a \(4\times4\) grid. Specifically, the data-generating distribution is
\[
P
=
\frac{1}{16}
\sum_{j=1}^{16}
\mathcal N(\mu_j,0.2I_2),
\qquad
\mu_j\in\{-9,-3,3,9\}^2,
\]
and we generate \(n=5{,}000\) observations from \(P\).

For simulation study, we compare four fixed update schedules, \(1_D{:}1_G\), \(5_D{:}1_G\), \(1_D{:}5_G\), and \(5_D{:}5_G\), with the proposed adaptive schedule. Here, \(k_D{:}k_G\) means that \(k_D\) discriminator updates are followed by \(k_G\) generator updates at each outer training iteration. For the adaptive schedule, we set the discriminator and generator discrepancy thresholds to $(a_D,b_G)=(0.01,0.05)$ and $(\alpha_D,\alpha_G)=(0.1,0.1)$. 

The main experiments use original GAN objective. For a discriminator score \(D\), the discriminator and generator losses are given by
\[
\begin{split}
    \mathcal L_D
&=
\mathbb E_{ P}
\left[
\operatorname{softplus}\{-D_{\theta_d}(X)\}
\right]
+
\mathbb E_{ Q_{\theta_g}}
\left[
\operatorname{softplus}\{D_{\theta_d}(Y)\}
\right],\\
\mathcal L_G
&=
\mathbb E_{ Q_{\theta_g}}
\left[
\operatorname{softplus}\{-D_{\theta_d}(Y)\}
\right],
\end{split}
\]
respectively.  Each method is trained for \(6{,}000\) outer iterations. Additional details regarding the network architectures, minibatch size, learning rates, betting fractions, and the minimum and maximum numbers of adaptive updates are provided in the Supplementary Material.

Figure~\ref{fig:gaussian16} demonstrates that the quality of the generated samples depends strongly on the relative numbers of discriminator and generator updates. The \(1_D{:}1_G\), \(5_D{:}1_G\), and \(1_D{:}5_G\) schedules fail to recover the multimodal structure, whereas \(5_D{:}5_G\) captures such multimodal structure. The proposed adaptive schedule most faithfully recovers all \(16\) modes without requiring a prespecified update ratio.

\begin{figure}[!t]
    \centering
    % Row 1
    \begin{subfigure}{0.47\linewidth}
        \centering
        \includegraphics[
            width=\linewidth,
            height=3.2cm,
            keepaspectratio
        ]{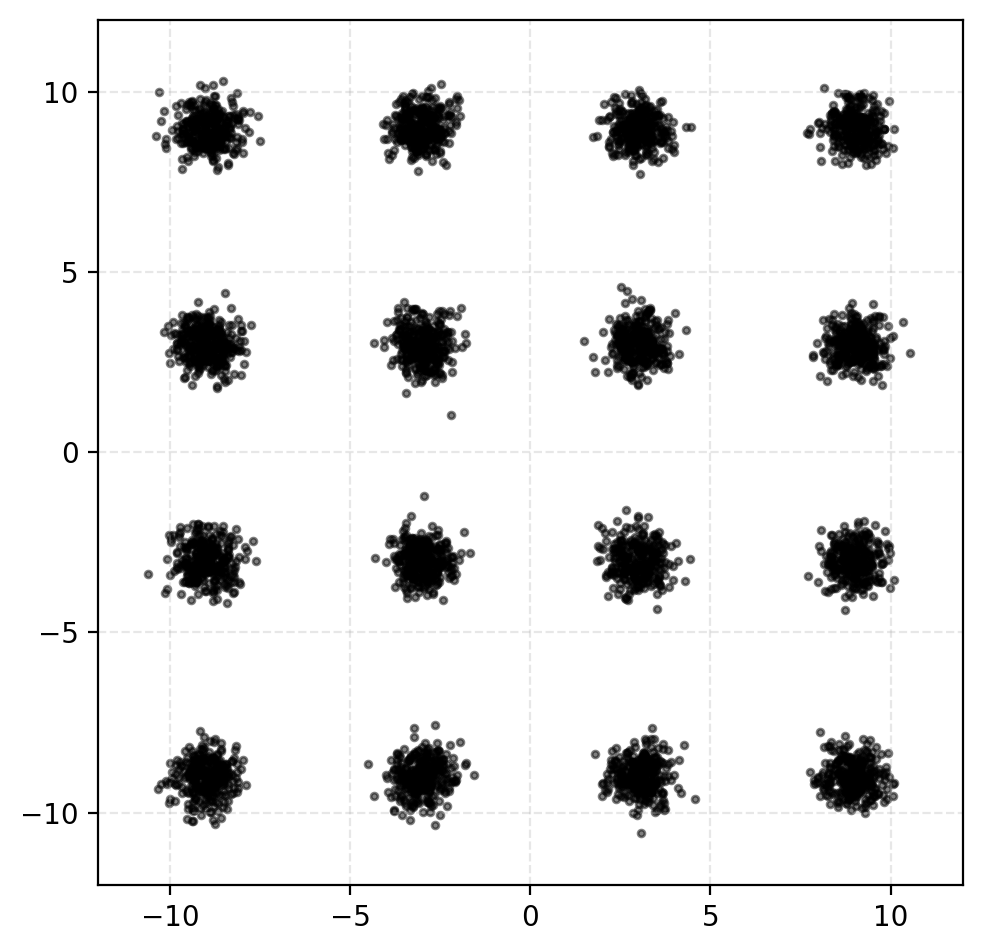}
        \caption{True data}
    \end{subfigure}
    \hfill
    \begin{subfigure}{0.47\linewidth}
        \centering
        \includegraphics[
            width=\linewidth,
            height=3.2cm,
            keepaspectratio
        ]{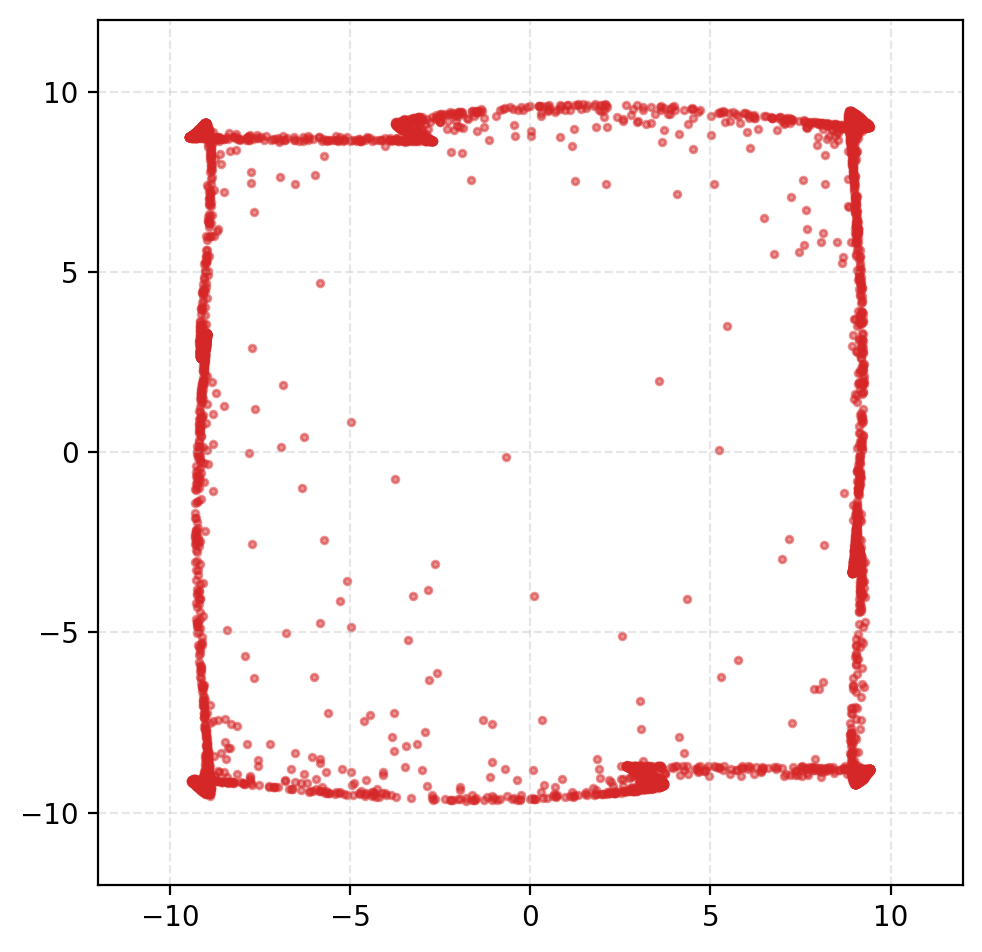}
        \caption{Fixed schedule ($1{:}1$)}
    \end{subfigure}

    \vspace{0.1em}

    % Row 2
    \begin{subfigure}{0.47\linewidth}
        \centering
        \includegraphics[
            width=\linewidth,
            height=3.2cm,
            keepaspectratio
        ]{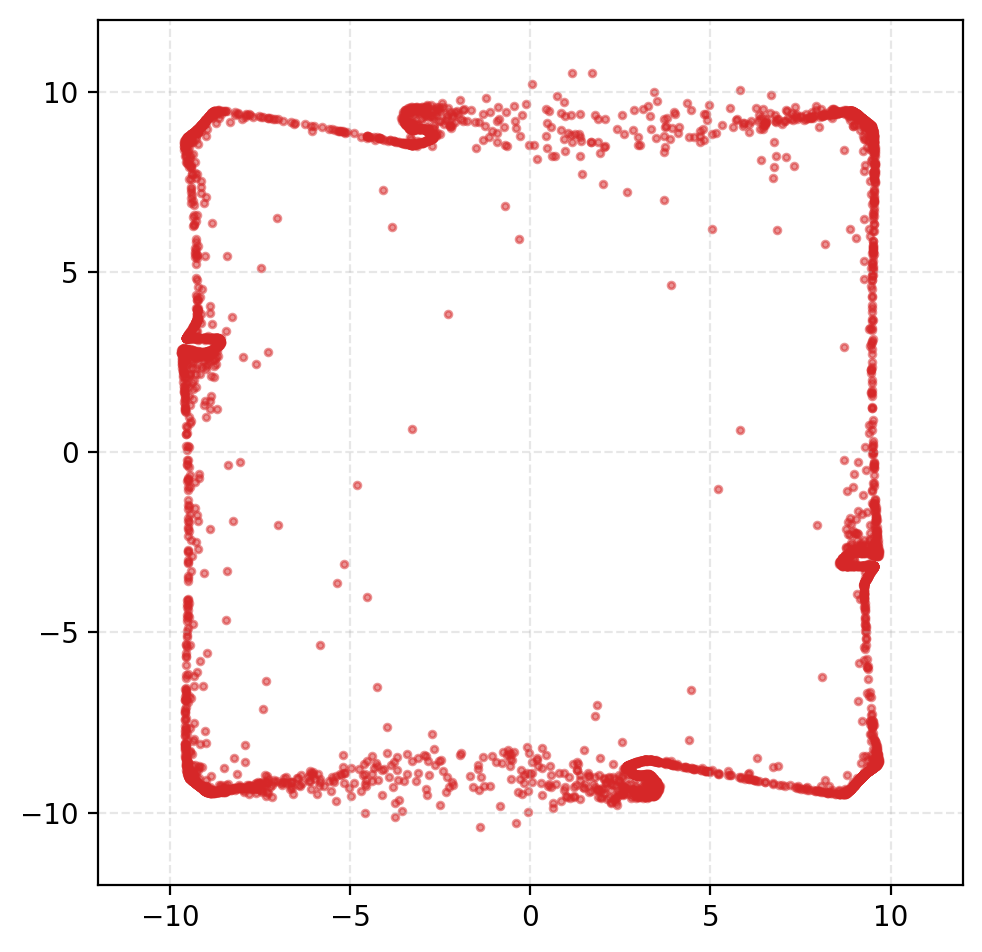}
        \caption{Fixed schedule ($5{:}1$)}
    \end{subfigure}
    \hfill
    \begin{subfigure}{0.47\linewidth}
        \centering
        \includegraphics[
            width=\linewidth,
            height=3.2cm,
            keepaspectratio
        ]{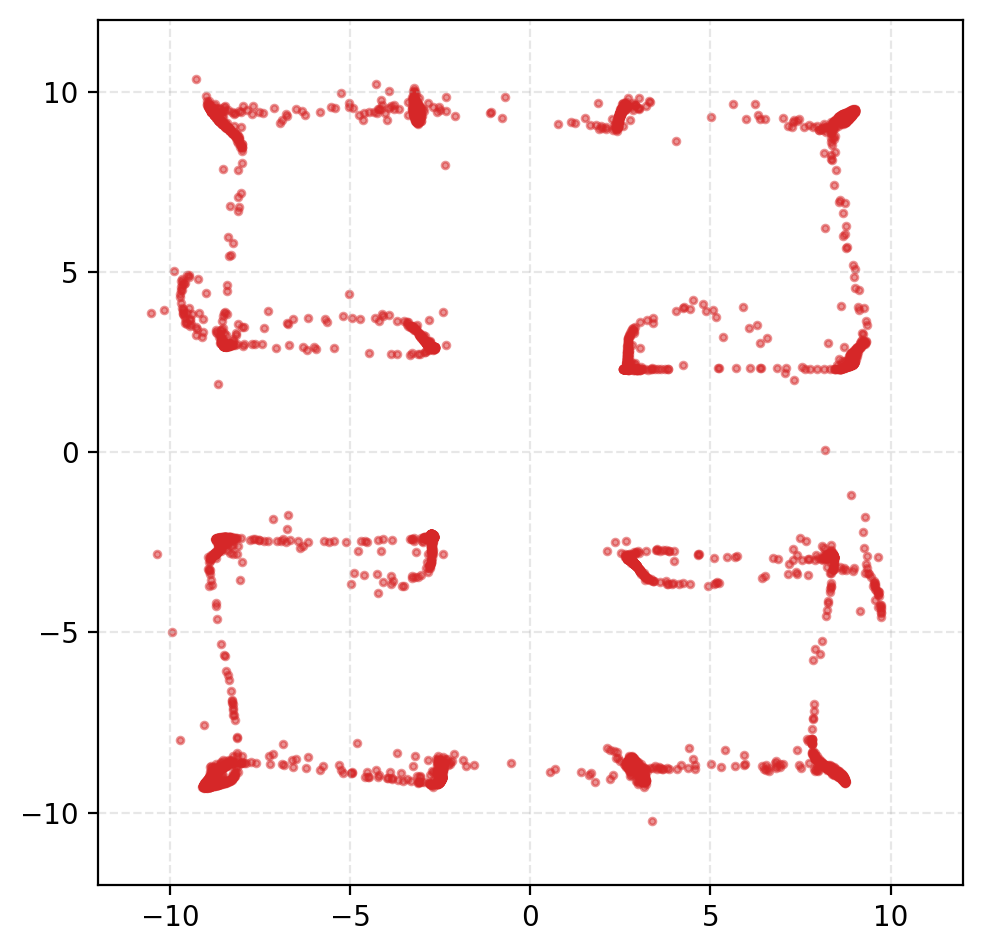}
        \caption{Fixed schedule ($1{:}5$)}
    \end{subfigure}

    \vspace{0.1em}

    % Row 3
    \begin{subfigure}{0.47\linewidth}
        \centering
        \includegraphics[
            width=\linewidth,
            height=3.2cm,
            keepaspectratio
        ]{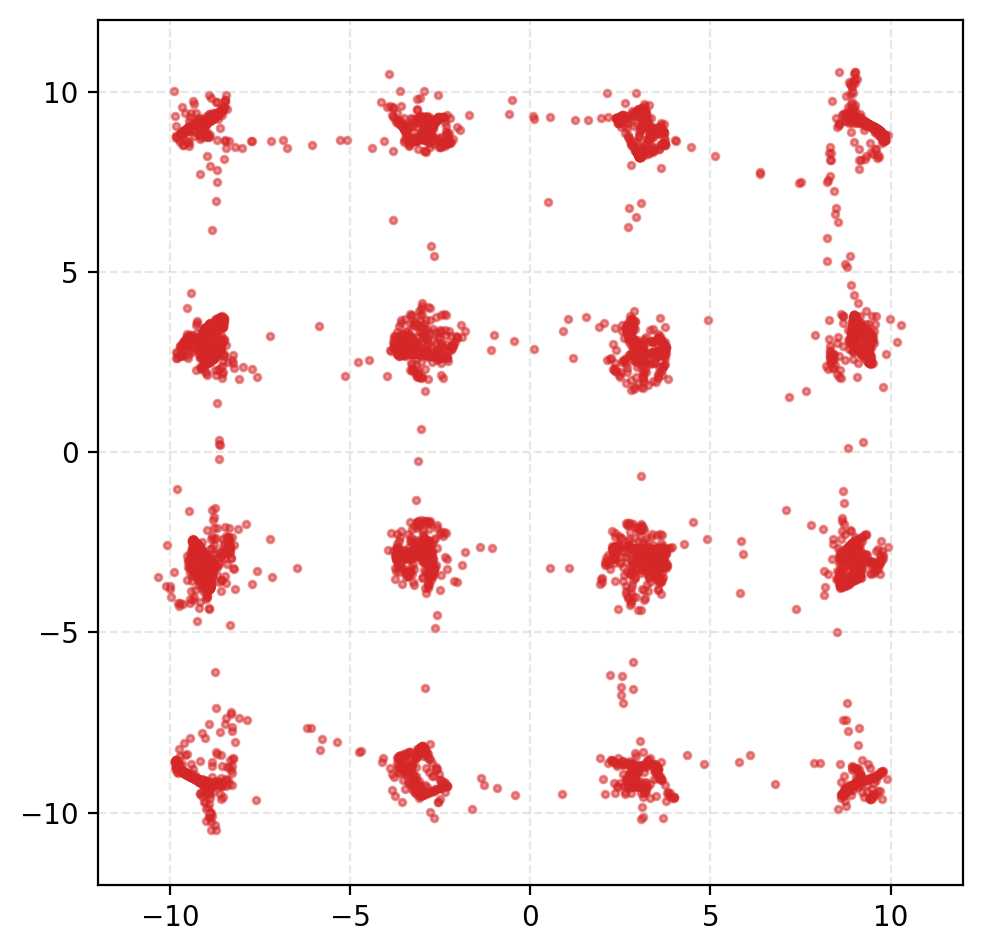}
        \caption{Fixed schedule ($5{:}5$)}
    \end{subfigure}
    \hfill
    \begin{subfigure}{0.47\linewidth}
        \centering
        \includegraphics[
            width=\linewidth,
            height=3.2cm,
            keepaspectratio
        ]{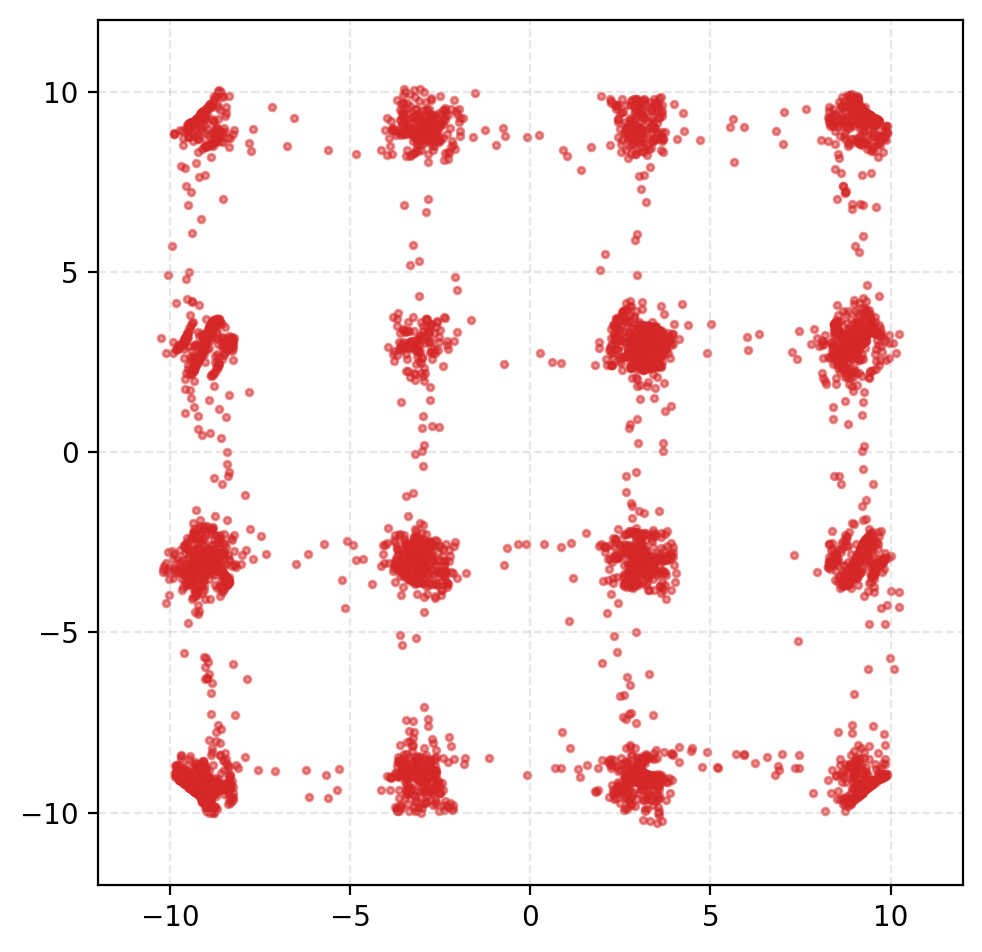}
        \caption{Adaptive schedule}
    \end{subfigure}

    \caption{
        Comparison of fixed and adaptive discriminator--generator update schedules on the 16-mode Gaussian mixture. The true data are shown as a reference.
    }
    \label{fig:gaussian16}
\end{figure}

Figure~\ref{fig:eprocess_trajectories} illustrates the dynamics of the discriminator--generator minimax game under the proposed adaptive update scheme. During approximately the first \(2{,}000\) outer iterations, both the discriminator and generator are often updated multiple times before the algorithm switches to the opposing phase, reflecting active
rebalancing between the two networks. In later iterations, the
discriminator phase typically reaches the maximum number of updates \(L_D\), whereas the generator phase often terminates after a single update. This pattern is consistent with convergence of the generator distribution toward the data distribution. As the two distributions become increasingly similar, the discriminator requires more updates but still fails to accumulate sufficient evidence of separation, while
the generator-phase e-process quickly indicates that the fixed
discriminator no longer certifies separation by at least \(b_G\).

\begin{figure}[!h]
    \centering
    \includegraphics[height=4cm, width=0.95\linewidth]{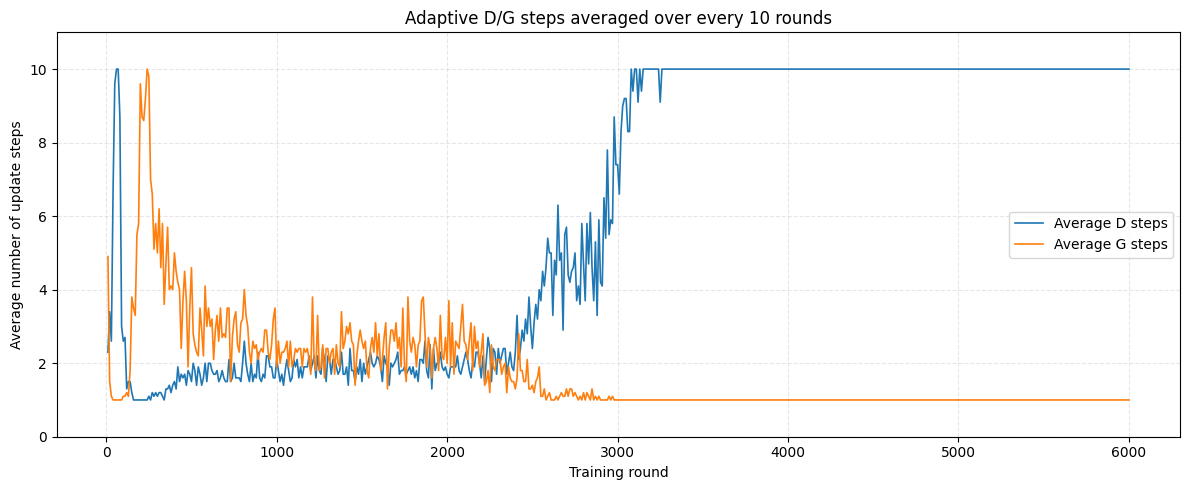}
    \caption{
    Adaptive numbers of discriminator and generator updates during training.
    The selected numbers of updates are recorded every 10 switches between
    the discriminator and generator phases.
    }
    \label{fig:eprocess_trajectories}
\end{figure}

For the original GAN objective, the fixed \(5_D{:}5_G\) schedule yields relatively good results. This schedule, however, does not carry over uniformly to other adversarial objectives. Fixed schedules that work well for the original GAN can behave quite differently under the hinge and Wasserstein losses. The proposed method instead adjusts the numbers of discriminator and generator updates using the two e-processes, and shows more consistent behavior across the objectives considered here. Results for the hinge and Wasserstein GAN losses are reported in the Supplementary Material.

\subsection{Benchmark Dataset: CIFAR-10}
\label{sec:benchmark_experiments}

We evaluate the proposed adaptive update strategy on CIFAR-10
\citep{krizhevsky2009learning}, a standard natural-image benchmark for assessing the quality of the generated samples. All experiments use the same ResNet-based SNGAN generator and discriminator following \citet{miyato2018spectral}, with spectral normalization applied to the discriminator.

We consider two commonly used adversarial objectives: the original GAN objective, referred to as softplus, and the hinge objective. The proposed method is compared with fixed update schedules of \(1{:}1\) and \(5{:}1\), as well as the Two Time-Scale Update Rule (TTUR) \citep{heusel2017ttur}. TTUR performs one discriminator update and one generator update per outer iteration while using different learning rates for the two networks. Performance is evaluated using Inception Score (IS) and Fr\'echet Inception Distance (FID).

Table~\ref{tab:benchmark_is_fid} compares the adaptive strategy with
the fixed-ratio and TTUR baselines. Each method is trained until the
generator has completed \(100{,}000\) updates. For each run, we report
the highest IS and the lowest FID observed over the every $5,000$ evaluation
checkpoints, and summarize the results as the mean and standard
deviation across five independent runs.

\begin{table}[!htbp]
\centering
\caption{CIFAR-10 results under the softplus and hinge GAN objectives.
Higher IS and lower FID are better, and the best mean value in each
loss group is shown in bold.}
\label{tab:benchmark_is_fid}
\footnotesize
\setlength{\tabcolsep}{4pt}
\renewcommand{\arraystretch}{1.06}
\begin{tabular}{@{}llcc@{}}
\toprule
GAN loss
& Training strategy
& IS \(\uparrow\)
& FID \(\downarrow\) \\
\midrule

\multirow{4}{*}{Softplus}
& Fixed \(1{:}1\)
& \(4.230 \pm 0.129\)
& \(83.407 \pm 6.443\) \\

& Fixed \(5{:}1\)
& \(8.072 \pm 0.043\)
& \(17.540 \pm 0.553\) \\

& TTUR
& \(7.139 \pm 0.494\)
& \(34.997 \pm 5.960\) \\

& Adaptive
& \(\mathbf{8.192 \pm 0.080}\)
& \(\mathbf{15.877 \pm 0.354}\) \\

\midrule

\multirow{4}{*}{Hinge}
& Fixed \(1{:}1\)
& \(4.416 \pm 0.387\)
& \(80.295 \pm 8.539\) \\

& Fixed \(5{:}1\)
& \(8.281 \pm 0.033\)
& \(13.542 \pm 0.316\) \\

& TTUR
& \(7.508 \pm 0.260\)
& \(31.106 \pm 5.629\) \\

& Adaptive
& \(\mathbf{8.330 \pm 0.079}\)
& \(\mathbf{13.208 \pm 0.357}\) \\

\bottomrule
\end{tabular}
\end{table}

Under both GAN objectives, the adaptive strategy yields the highest
mean IS and the lowest mean FID among the training strategies
considered. Further details on the architecture, preprocessing,
optimization, evaluation procedure, and adaptive parameters are
provided in the Supplementary Material.

\section{Conclusion}

We developed an e-process-guided framework for adaptively determining the numbers of discriminator and generator updates in GAN training. Rather than relying on a fixed update ratio, the proposed method uses sequential evidence to decide when to terminate the current optimization phase and switch to the opposing network. The resulting stopping rules are anytime-valid under repeated monitoring.

Our construction applies across GAN objectives through variational representations of distributional discrepancies. We introduced two e-variable constructions, established their validity under total variation neighborhoods, and characterized the corresponding oracle logarithmic e-power. The framework provides a statistically principled alternative to heuristic update schedules and may extend to other minimax problems requiring sequential assessment of competing components.

\section*{Acknowledgments}

The research of Hyunjoo Kim and Sehwan Kim was supported by the Global-Learning \& Academic Research Institution for Master's and PhD Students, and Postdocs (G-LAMP) Program of the National Research Foundation of Korea (NRF), funded by the Ministry of Education (No. RS-2025-25442252).

Guang Lin would like to thank the support of National Science Foundation (DMS-2533878, DMS-2053746, DMS-2134209, ECCS-2328241, CBET-2347401 and OAC-2311848), and U.S.~Department of Energy (DOE) Office of Science Advanced Scientific Computing Research program under the "Uncertainty Quantification for Multifidelity Operator Learning (MOLUcQ)" project (Project No. 81739), DE-SC0023161, the SciDAC LEADS Institute, and DOE–Fusion Energy Science, under grant number: DE-SC0024583.

\bibliography{aaai2027}

%%%%%%%%%%%%%%%%%%%%%%%%%%%%%%%%%%%%%%%%%%%%%%%%%%%%%%%%%%%%

\clearpage
\appendix

\section*{Supplementary Material}

\section{Complete Proofs}
\label{app:proofs}

\subsection{Proof of Proposition~\ref{prop:discriminator-evalue}}

\begin{proof} For the separable-score e-variable, define
\begin{equation}\label{d_D}
    d_D(x)=\sigma\{D(x)\},
\qquad 0\leq d_D(x)\leq 1.
\end{equation}
By the independence of \(X\) and \(Y\),
\[
\begin{aligned}
\mathbb E_{P\times Q}
\left[
4\sigma\{D(X)\}\sigma\{-D(Y)\}
\right]
&=
4\mathbb E_P[g(X)]
\mathbb E_Q[1-g(Y)].
\end{aligned}
\]
Let $u=\mathbb E_P[g(X)],$ and $v=\mathbb E_Q[1-g(Y)].$ By definition of $TV(\cdot,\cdot)$, we can derive
\[
u+v
\leq
1+\operatorname{TV}(P,Q)
\leq
1+a_D.
\]
And by \(4uv\leq(u+v)^2\),
\[
\mathbb E_{P\times Q}
\left[
4\sigma\{D(X)\}\sigma\{-D(Y)\}
\right]
\leq
(1+a_D)^2.
\]
It follows that $\mathbb E_{P\times Q}
\left[
E_D^{\mathrm{sep}}(X,Y)
\right]
\leq 1.$

For the difference-score e-variable, define
\[
h_D(x,y)=\sigma\{D(x)-D(y)\}.
\]
Using the identity $\sigma(z)+\sigma(-z)=1,$ we have
\[
\mathbb E_{Q\times P}[h(X,Y)]
=
\mathbb E_{P\times Q}[h(Y,X)]
=
1-\mathbb E_{P\times Q}[h(X,Y)].
\]
Hence,
\[
\begin{aligned}
2\mathbb E_{P\times Q}[h(X,Y)]-1
&=
\mathbb E_{P\times Q}[h(X,Y)]
-
\mathbb E_{Q\times P}[h(X,Y)]\\
&\leq
\operatorname{TV}(P\times Q,Q\times P),
\end{aligned}
\]
where the inequality follows from \(0\leq h\leq1\).

Let \(\delta=\operatorname{TV}(P,Q)\). By taking two independent
maximal couplings of \(P\) and \(Q\),
\[
\operatorname{TV}(P\times Q,Q\times P)
\leq
1-(1-\delta)^2
=
2\delta-\delta^2.
\]
Under \(H_0^D(a_D)\), we have \(\delta\leq a_D\). Since the function
\(t\mapsto 2t-t^2\) is non-decreasing on \([0,1]\),
\[
\operatorname{TV}(P\times Q,Q\times P)
\leq
2\delta-\delta^2
\leq
2a_D-a_D^2.
\]
It follows that
\[
2\mathbb E_{P\times Q}
\left[
\sigma\{D(X)-D(Y)\}
\right]
\leq
1+2a_D-a_D^2,
\]
and therefore
\[
\mathbb E_{P\times Q}
\left[
E_D^{\mathrm{diff}}(X,Y)
\right]
\leq 1.
\]
Thus, both \(E_D^{\mathrm{sep}}\) and \(E_D^{\mathrm{diff}}\) are
e-variables under \(H_0^D(a_D)\).
\end{proof}

\subsection{Proof of Proposition~\ref{prop:oracle-epower}}

\begin{proof}
For the separable-score e-variable, using \(d_D\) defined in Equation~\eqref{d_D}, we have
\begin{align*}
\Gamma_{P,Q}^{\mathrm{sep}}(D)
&=
\log 4-2\log(1+a_D)\\
&\quad+
\int
\left[
p(x)\log d_D(x)
+
q(x)\log\{1-d_D(x)\}
\right]d\mu(x).
\end{align*}
For each \(x\), the integrand is maximized at $d_D^\star(x)
=
\frac{p(x)}{p(x)+q(x)},$ which is equivalent to
\[
D_{\mathrm{sep}}^\star(x)
=
\log\frac{p(x)}{q(x)}.
\]
Substituting this optimizer gives
\[
\sup_D\Gamma_{P,Q}^{\mathrm{sep}}(D)
=
2\operatorname{JS}(P,Q)
-
2\log(1+a_D).
\]

For the difference-score e-variable, let
\[
R=P\times Q,
\qquad
S=Q\times P,
\qquad
T_D(x,y)=D(x)-D(y).
\]
By symmetry, 
\[
\mathbb E_R[\log\sigma\{T_D(X,Y)\}]
=
\mathbb E_S[\log\sigma\{-T_D(X,Y)\}],
\]
and hence
\begin{align*}
\Gamma_{P,Q}^{\mathrm{diff}}(D)
&=
\log 2-\log(1+2a_D-a_D^2)\\
&\quad+
\frac12
\left[
\mathbb E_R[\log\sigma\{T_D(X,Y)\}]
+
\mathbb E_S[\log\sigma\{-T_D(X,Y)\}]
\right].
\end{align*}
The logistic variational objective is maximized at
\[
T_D^\star(x,y)
=
\log\frac{p(x)q(y)}{q(x)p(y)}
=
\log\frac{p(x)}{q(x)}
-
\log\frac{p(y)}{q(y)}.
\]
This optimizer is attained by
\[
D_{\mathrm{diff}}^\star(x)
=
\log\frac{p(x)}{q(x)}+c,
\qquad c\in\mathbb R.
\]
Therefore,
\[
\sup_D\Gamma_{P,Q}^{\mathrm{diff}}(D)
=
\operatorname{JS}(P\times Q,Q\times P)
-
\log(1+2a_D-a_D^2).
\]

Now let $\delta=\operatorname{TV}(P,Q).$ The lower bound relating the Jensen--Shannon divergence to total variation gives
\[
\operatorname{JS}(P,Q)\geq h(\delta).
\]
Since \(h\) is increasing on \([0,1]\), under
\(H_1^D(b_D)\),
\[
\operatorname{JS}(P,Q)\geq h(b_D).
\]
Thus,
\[
\sup_D\Gamma_{P,Q}^{\mathrm{sep}}(D)
\geq
2\{h(b_D)-\log(1+a_D)\},
\]
which is positive whenever
\[
h(b_D)>\log(1+a_D).
\]

Moreover, marginalization cannot increase total variation, so
\[
\operatorname{TV}(P\times Q,Q\times P)
\geq
\operatorname{TV}(P,Q)
\geq b_D.
\]
Consequently,
\[
\operatorname{JS}(P\times Q,Q\times P)
\geq h(b_D),
\]
and hence
\[
\sup_D\Gamma_{P,Q}^{\mathrm{diff}}(D)
\geq
h(b_D)-\log(1+2a_D-a_D^2),
\]
which is positive whenever
\[
h(b_D)>
\log(1+2a_D-a_D^2).
\]
This proves the result.
\end{proof}

\subsection{Proof of Proposition~\ref{prop:generator-evalue}}

\begin{proof}
Let $\widetilde d_t(x)=\sigma\{\widetilde D_t(x)\}.$ For a fixed generator iteration \(\ell\), the operational null
\eqref{HP:generator} implies that
\[
\Delta_{t,\ell}(\widetilde D_t)
=
\mathbb E_P[\widetilde d_t(X)]
-
\mathbb E_{Q_{t,\ell}}[\widetilde d_t(Y)]
\geq b_G.
\]

For the separable-score e-variable, let
\[
u=\mathbb E_{Q_{t,\ell}}[\widetilde d_t(Y)],
\qquad
v=\mathbb E_P[1-\widetilde d_t(X)].
\]
Then
\[
u+v
=
1-\Delta_{t,\ell}(\widetilde D_t)
\leq
1-b_G.
\]
By the independence of \(X\) and \(Y\), and the inequality
\(4uv\leq (u+v)^2\),
\begin{align*}
&\mathbb E_{P\times Q_{t,\ell}}
\left[
4\sigma\{\widetilde D_t(Y)\}
\sigma\{-\widetilde D_t(X)\}
\right]\\
&\qquad=
4\mathbb E_{Q_{t,\ell}}[\widetilde d_t(Y)]
\mathbb E_P[1-\widetilde d_t(X)]\\
&\qquad=
4uv
\leq
(u+v)^2
\leq
(1-b_G)^2.
\end{align*}
Therefore, $\mathbb E_{P\times Q_{t,\ell}}
\left[
E_{\widetilde D_t}^{G,\mathrm{sep}}(X,Y)
\right]
\leq 1.$

For the difference-score e-variable, note that, for any
\(r,s\in\mathbb R\),
\[
\sigma(s-r)
\leq
\sigma(s)+\sigma(-r).
\]
Indeed, direct calculation gives
\[
\sigma(s)+\sigma(-r)-\sigma(s-r)
=
\frac{
e^{2r+s}+e^{r+s}+e^r+e^{2s}
}{
(1+e^r)(e^r+e^s)(1+e^s)
}
\geq 0.
\]
Applying this inequality with
\(r=\widetilde D_t(X)\) and \(s=\widetilde D_t(Y)\), we obtain
\begin{align*}
&\mathbb E_{P\times Q_{t,\ell}}
\left[
\sigma\{\widetilde D_t(Y)-\widetilde D_t(X)\}
\right]\\
&\qquad\leq
\mathbb E_{Q_{t,\ell}}[\widetilde d_t(Y)]
+
\mathbb E_P[1-\widetilde d_t(X)]\\
&\qquad=
1-\Delta_{t,\ell}(\widetilde D_t)
\leq
1-b_G.
\end{align*}
Consequently,
\[
\mathbb E_{P\times Q_{t,\ell}}
\left[
E_{\widetilde D_t}^{G,\mathrm{diff}}(X,Y)
\right]
\leq 1.
\]
Thus, both quantities are e-variables under
\eqref{HP:generator}.
\end{proof}

\section{Detailed Description of the Adaptive Training Algorithm}\label{app:algorithm}

Algorithm~\ref{alg:adaptive_full} summarizes the adaptive training procedure used in the CIFAR-10 experiments. During each discriminator (generator) phase, the corresponding network is updated using training mini-batches while parameter gradients are disabled for the other network.

After every update, a fresh evaluation mini-batch is used to compute the mini-batch e-value according to the corresponding definition in the main text. The updated network is temporarily set to evaluation mode when computing this statistic. The cumulative e-process is then updated using product aggregation, and the current phase terminates once the stopping criterion is satisfied or the maximum number of updates is reached. 

\begin{algorithm}[!t]
\small
\caption{E-Process-Guided Adaptive GAN Training}
\label{alg:adaptive_full}
\begin{algorithmic}[1]

\Require  Generator \(G\), discriminator \(D\), stopping levels \(\alpha_D,\alpha_G\), betting fractions \(\rho_D,\rho_G\), minimum update counts \(L_D^{\min},L_G^{\min}\), maximum update counts \(L_D,L_G\), discriminator and generator radius parameters \(a_D,b_G\), total generator-update budget \(T_G\)

\State Initialize \(t_G\gets0\)

\While{\(t_G<T_G\)}

    \Statex \textbf{Discriminator phase}
    \State \(\mathcal E_D\gets1\)

    \For{\(\ell=1,\ldots,L_D\)}
        \State Update \(D\) using a training mini-batch
        \State Draw a fresh evaluation batch
        \(\{(x_i,\tilde x_i)\}_{i=1}^m\), where
        \(\tilde x_i=G(z_i)\)
        
        Compute 
        \[
        E_D^{\mathrm{mb}}
        =
        \prod_{i=1}^m
        \left[
        \frac12+
        \frac{
        2\sigma\{D(x_i)\}\sigma\{-D(\tilde x_i)\}
        }{(1+a_D)^2}
        \right]
        \]
        
        \State
        \(\mathcal E_D\gets
        \mathcal E_D(1-\rho_D+\rho_D E_D^{\mathrm{mb}})\)

        \If{\(\ell\ge L_D^{\min}\) and
        \(\mathcal E_D>1/\alpha_D\)}
            \State \textbf{break}
        \EndIf
    \EndFor

    \Statex \textbf{Generator phase}
    \State \(\mathcal E_G\gets1\)

    \For{\(\ell=1,\ldots,L_G\)}
        \If{\(t_G\ge T_G\)}
            \State \textbf{break}
        \EndIf

        \State Update \(G\) using a latent training mini-batch;
        \(t_G\gets t_G+1\)
        \State Draw a fresh evaluation batch
        \(\{(x_i,\tilde x_i)\}_{i=1}^m\), where
        \(\tilde x_i=G(z_i)\)
        \State Compute
        \[
        E_G^{\mathrm{mb}}
       =
        \prod_{i=1}^m
        \left[
        \frac12+
        \frac{
        2\sigma\{D(\tilde x_i)\}\sigma\{-D(x_i)\}
        }{(1-b_G)^2}
        \right]
        \]
        \State
        \(\mathcal E_G\gets
        \mathcal E_G(1-\rho_G+\rho_G E_G^{\mathrm{mb}})\)

        \If{\(\ell\ge L_G^{\min}\) and
        \(\mathcal E_G>1/\alpha_G\)}
            \State \textbf{break}
        \EndIf
    \EndFor

\EndWhile

\end{algorithmic}
\end{algorithm}

\section{Additional Details for Mixture of 16 Gaussians Experiments} \label{app:gaussian_details}

The latent dimension is 16. Both the generator and discriminator are
multilayer perceptrons with four hidden layers of width 200. The generator uses hyperbolic-tangent hidden activations, while the discriminator uses leaky-ReLU activations. Both networks are optimized using Adam with learning rate
\(2\times10^{-4}\) and hyper parameter $(\beta_1,\beta_2)=(0,0.9).$ The mini-batch size is 100, and each run consists of \(6{,}000\) outer training rounds.

For the adaptive method, the minimum and maximum numbers of updates for both networks are 1 and 10, respectively. The adaptive parameters are $(\alpha_D,\alpha_G)=(0.1,0.1)$, $(a_D,b_G)=(0.01,0.05)$ and $(\rho_D,\rho_G)=(0.5,0.5)$.

\subsection{Results for Alternative Loss Functions}

The main manuscript presents the representative result obtained using the original GAN loss. Additional experiments using the hinge and Wasserstein objectives are reported in this section.

For the hinge GAN loss, the discriminator and generator are trained according to
\[
\begin{split}
\min_D \,
& \mathbb E_{X\sim P}
\left[
\max\{0,1-D(X)\}
\right]
+
\mathbb E_{Y\sim Q_\theta}
\left[
\max\{0,1+D(Y)\}
\right],\\
\min_\theta \,
& -\mathbb E_{Y\sim Q_\theta}
\left[
D(Y)
\right].
\end{split}
\]

Figure~\ref{fig:sup_hinge} shows that, under the hinge GAN loss, the quality of the generated samples depends strongly on the relative numbers of discriminator and generator updates. The $1_D:1_G$ and $5_D:5_G$ schedules recover the multimodal structure reasonably well, whereas the commonly used $5_D:1_G$ schedule fails to do so. The adaptive schedule most faithfully recovers all 16 modes without requiring a prespecified update ratio.

\begin{figure}[!htbp]
    \centering
    \includegraphics[ width=1\linewidth]{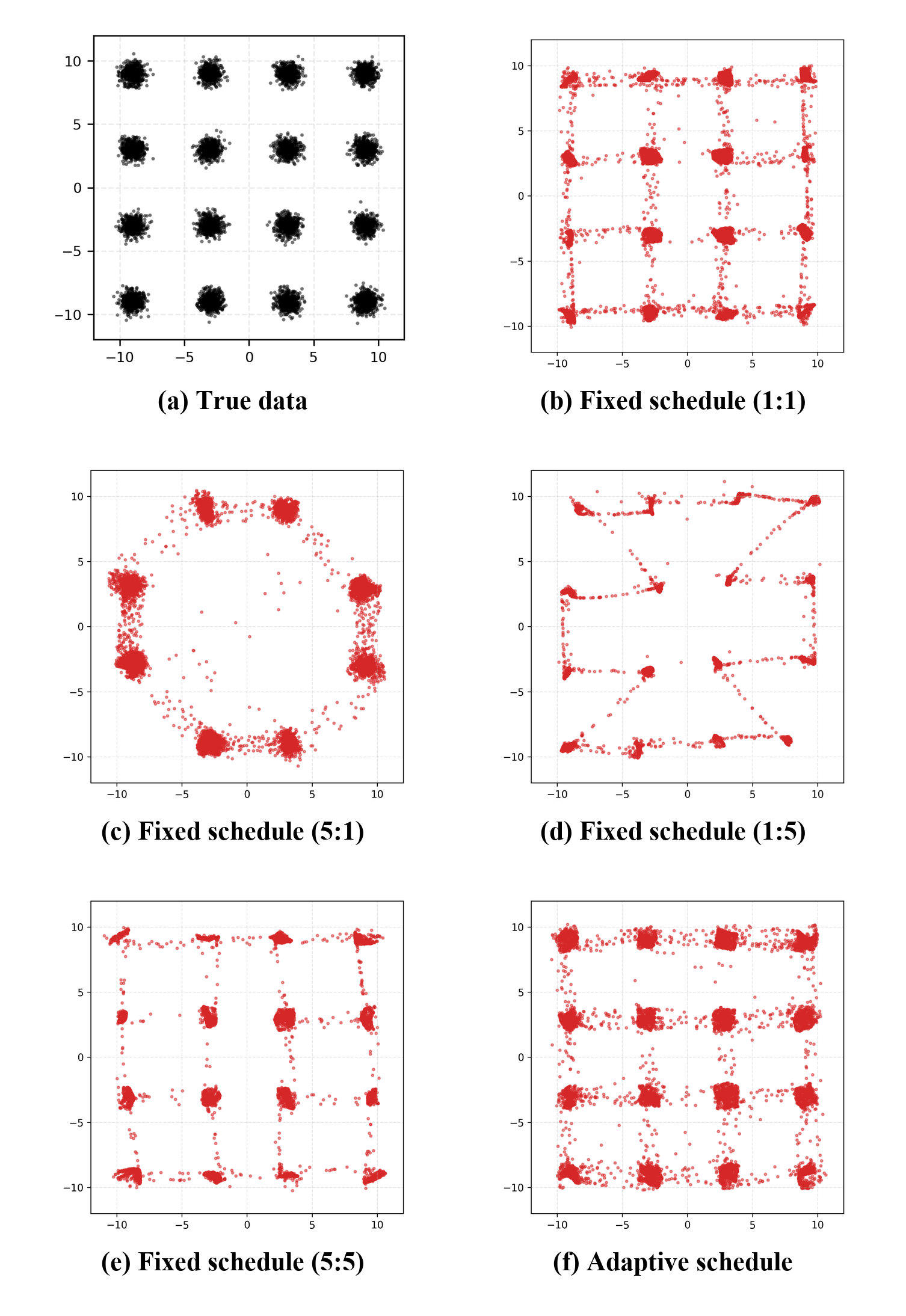}
    \caption{Comparison of fixed and adaptive discriminator--generator update schedules under the hinge GAN loss on the 16-mode Gaussian mixture.}
    \label{fig:sup_hinge}
\end{figure}

\FloatBarrier

\needspace{3\baselineskip}
For the Wasserstein GAN-GP loss, the critic and generator are trained according to
\[
\begin{split}
\max_D \,
& \mathbb E_{X\sim P}
\left[
D(X)
\right]
-
\mathbb E_{Y\sim Q_\theta}
\left[
D(Y)
\right]
-
\mathcal L_{\mathrm{GP}},\\
\min_\theta \,
& -\mathbb E_{Y\sim Q_\theta}
\left[
D(Y)
\right].
\end{split}
\]

Figure~\ref{fig:sup_wasserstein} further demonstrates that the
performance of the Wasserstein GAN is highly sensitive to the relative
numbers of critic and generator updates. The $1_D:1_G$ and $5_D:1_G$ schedules
generate diffuse grid-like patterns rather than samples concentrated
around the target modes. Increasing the relative number of generator
updates to $1_D:5_G$ leads to severe mode collapse, whereas the $5_D:5_G$
schedule recovers the multimodal structure only partially and with
noticeable distortion. In contrast, the proposed adaptive schedule
most accurately recovers the 16-mode structure without requiring a
prespecified update ratio.

\begin{figure}[!htbp]
    \centering
    \includegraphics[ width=1\linewidth]{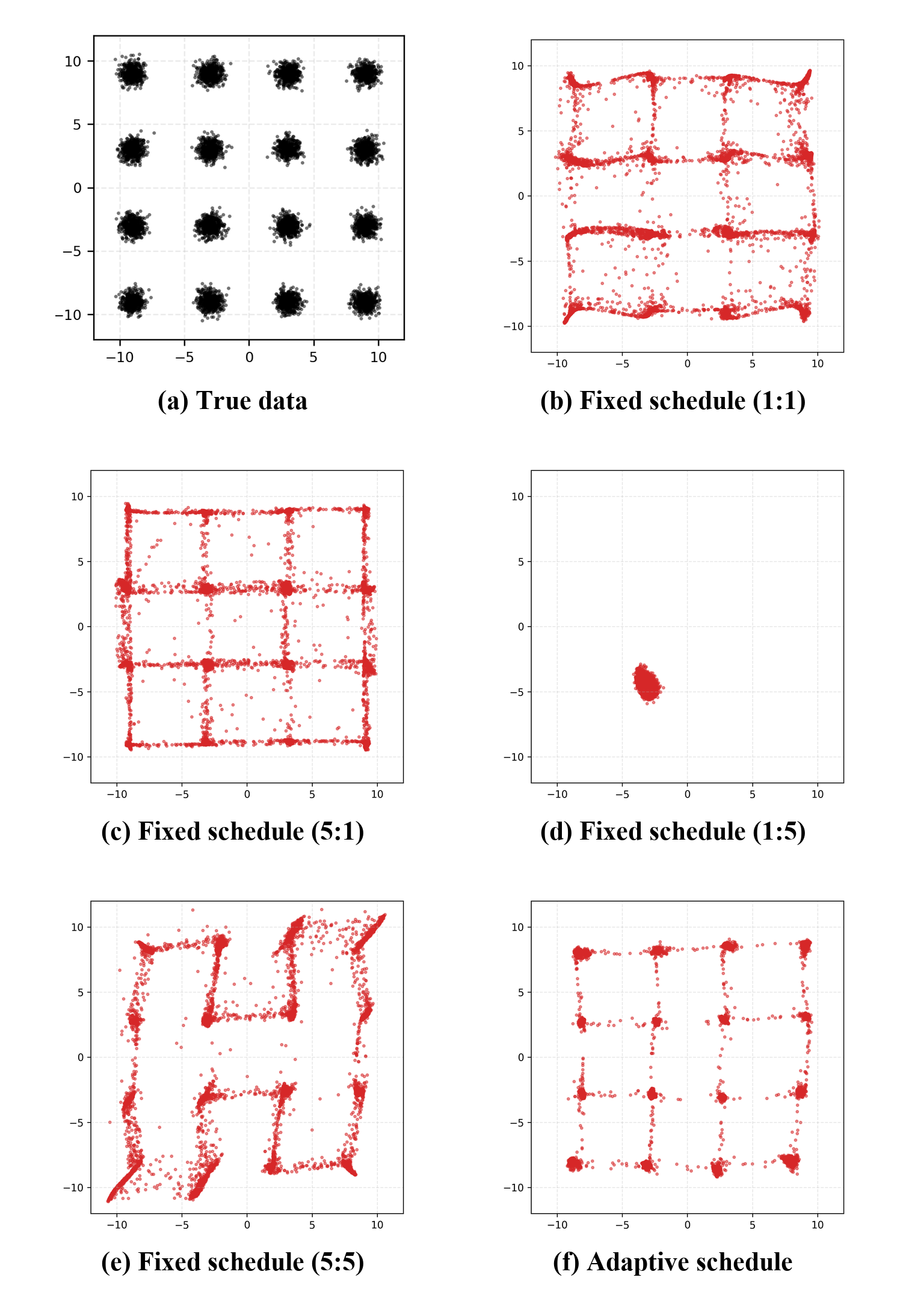}
    \caption{Comparison of fixed and adaptive discriminator--generator
update schedules under the Wasserstein GAN loss on the 16-mode
Gaussian mixture. }
    \label{fig:sup_wasserstein}
\end{figure}

\FloatBarrier

\section{Additional Details for the CIFAR-10 Experiments}
\label{app:cifar_details}

\subsection{Dataset and Preprocessing}

We use the CIFAR-10 training set. Images are converted to tensors and
normalized channel-wise using mean \(0.5\) and standard deviation \(0.5\).
During training, an independent positive uniform perturbation is added to
each normalized input value:
\[
\widetilde X
=
X+\varepsilon,
\qquad
\varepsilon_i\sim\operatorname{Uniform}[0,1/128).
\]
No additional data augmentation is applied. The perturbation is used only
in the training data pipeline and is not applied to generated images
during IS/FID evaluation.

\subsection{Network Architecture}

All CIFAR-10 experiments use the same ResNet-based SNGAN generator and
discriminator. The generator receives
\[
Z\sim\mathcal N(0,I_{128})
\]
and produces \(32\times32\) RGB images. It consists of an initial linear
layer followed by three residual upsampling blocks with 256 channels.
The output layer applies batch normalization, ReLU, a \(3\times3\)
convolution, and a hyperbolic-tangent activation.

The discriminator consists of one optimized residual downsampling block,
followed by three residual blocks and global sum pooling. Spectral
normalization is applied to all discriminator convolutional and linear
layers. The same architecture and initialization procedure are used for all training strategies.

\subsection{Shared Training and Evaluation Settings}

The settings shared across all CIFAR-10 experiments are summarized in
Table~\ref{tab:cifar_shared_settings}.

\begin{table}[!htbp]
%\centering
\caption{Shared CIFAR-10 training settings.}
\label{tab:cifar_shared_settings}
\footnotesize
\setlength{\tabcolsep}{4pt}
\renewcommand{\arraystretch}{1.05}
\begin{tabular}{@{}ll@{}}
\toprule
Component & Setting \\
\midrule
Dataset & CIFAR-10 training set \\
Image resolution & \(32\times32\) RGB \\
Latent distribution & \(\mathcal N(0,I_{128})\) \\
Optimizer & Adam \\
Adam coefficients & \((\beta_1,\beta_2)=(0,0.9)\) \\
Discriminator batch size & 64 \\
Generator batch size & 128 \\
Training budget & 100,000 generator updates \\
Learning-rate schedule & Linear decay over generator updates \\
Random seeds & \(0,1,2,3,4\) \\
\bottomrule
\end{tabular}
\end{table}

For all compared training strategies, the discriminator and generator
learning rates are linearly decayed according to the cumulative number of
generator updates and reach zero at the end of the
100,000-generator-update training budget. Consequently, the total number
of discriminator updates may differ across training strategies, while
all methods use the same generator-update-based decay horizon.

\subsection{Method-Specific Configurations}

Table~\ref{tab:cifar_method_settings} summarizes the learning rates and
update schedules used by the compared methods.

\begin{table}[!htbp]
\centering
\caption{Method-specific CIFAR-10 training configurations. The listed
learning rates are their initial values.}
\label{tab:cifar_method_settings}
\scriptsize
\setlength{\tabcolsep}{2pt}
\renewcommand{\arraystretch}{1.04}
\begin{tabular}{@{}lcccc@{}}
\toprule
Method
& Update schedule
& Initial \(\eta_D\)
& Initial \(\eta_G\)
& Scheduling rule \\
\midrule
Fixed \(1{:}1\)
& \(1_D{:}1_G\)
& \(2{\times}10^{-4}\)
& \(2{\times}10^{-4}\)
& Fixed ratio \\

Fixed \(5{:}1\)
& \(5_D{:}1_G\)
& \(2{\times}10^{-4}\)
& \(2{\times}10^{-4}\)
& Fixed ratio \\

TTUR
& \(1_D{:}1_G\)
& \(4{\times}10^{-4}\)
& \(1{\times}10^{-4}\)
& Unequal learning rates \\

Adaptive
& \(1\text{--}10\) each
& \(2{\times}10^{-4}\)
& \(2{\times}10^{-4}\)
& E-process stopping \\
\bottomrule
\end{tabular}
\end{table}

For the adaptive method, the discriminator and generator independently
perform between 1 and 10 updates during their respective training phases.
The betting fractions are fixed at
\[
\rho_D=\rho_G=0.5.
\]
The resulting numbers of discriminator and generator updates are selected
dynamically and do not constitute a fixed update ratio.

\FloatBarrier

\subsection{Adaptive-Parameter Selection}
\label{app:adaptive_selection}

We examine adaptive scheduling configurations separately for the
softplus and hinge objectives because their discriminator-score dynamics
differ during training. All candidate configurations use the same
generator-update budget, evaluation protocol, and five random seeds.

The generator-side parameters were fixed at
\[
\alpha_G=0.10,\qquad a_G=0.30,
\]
based on preliminary experiments that produced stable generator-update
behavior. They were not included in the systematic search reported here.
To keep the computational cost manageable, we varied only the
discriminator-side parameters over
\[
\alpha_D\in\{0.05,0.10\},
\qquad
a_D\in\{0.10,0.15\},
\]
resulting in four evaluated combinations. The betting fractions were
fixed at
\[
\rho_D=\rho_G=0.5.
\]

For each seed and candidate configuration, the highest IS and lowest FID
were selected independently over the common set of reported evaluation
checkpoints. Thus, the two metrics may correspond to different model
checkpoints. Results are reported as the mean and sample standard
deviation across seeds \(0,1,2,3,4\).

\begin{table}[!htbp]
\centering
\caption{Adaptive scheduling configurations evaluated for the softplus
objective. For each seed, the highest IS and lowest FID are selected
independently across evaluation checkpoints.}
\label{tab:softplus_candidates}
\footnotesize
\setlength{\tabcolsep}{3pt}
\renewcommand{\arraystretch}{1.05}
\begin{tabular}{@{}cccc@{}}
\toprule
\(\alpha_D\)
& \(a_D\)
& IS \(\uparrow\)
& FID \(\downarrow\) \\
\midrule
\(\mathbf{0.10}\)
& \(\mathbf{0.10}\)
& \(\mathbf{8.192\pm0.080}\)
& \(\mathbf{15.877\pm0.354}\) \\

0.10
& 0.15
& \(8.167\pm0.071\)
& \(15.762\pm0.639\) \\

0.05
& 0.10
& \(8.112\pm0.054\)
& \(15.589\pm0.418\) \\

0.05
& 0.15
& \(8.113\pm0.037\)
& \(16.043\pm0.313\) \\
\bottomrule
\end{tabular}
\end{table}

For the softplus objective, mean IS was used as the primary selection
criterion, subject to the selected configuration also achieving a lower
mean FID than the Fixed \(5{:}1\) baseline. Based on this criterion, we
selected
\[
\alpha_D=0.10,\qquad a_D=0.10,
\]
which achieved the highest mean IS among the evaluated configurations,
with an IS of \(8.192\pm0.080\) and an FID of
\(15.877\pm0.354\). Although the configuration
\((\alpha_D,a_D)=(0.05,0.10)\) achieved a slightly lower mean FID, its
mean IS was lower.

\begin{table}[!t]
\centering
\caption{Adaptive scheduling configurations evaluated for the hinge
objective. For each seed, the highest IS and lowest FID are selected
independently across evaluation checkpoints.}
\label{tab:hinge_candidates}
\footnotesize
\setlength{\tabcolsep}{3pt}
\renewcommand{\arraystretch}{1.05}
\begin{tabular}{@{}cccc@{}}
\toprule
\(\alpha_D\)
& \(a_D\)
& IS \(\uparrow\)
& FID \(\downarrow\) \\
\midrule
0.10
& 0.10
& \(8.289\pm0.050\)
& \(13.218\pm0.341\) \\

0.10
& 0.15
& \(8.327\pm0.101\)
& \(13.248\pm0.701\) \\

0.05
& 0.10
& \(8.313\pm0.061\)
& \(13.360\pm0.331\) \\

\(\mathbf{0.05}\)
& \(\mathbf{0.15}\)
& \(\mathbf{8.330\pm0.079}\)
& \(\mathbf{13.208\pm0.357}\) \\
\bottomrule
\end{tabular}
\end{table}

For the hinge objective, the final configuration was selected based on
both mean IS and mean FID across the five random seeds. We selected
\[
\alpha_D=0.05,\qquad a_D=0.15,
\]
because it achieved both the highest mean IS,
\(8.330\pm0.079\), and the lowest mean FID,
\(13.208\pm0.357\), among the evaluated configurations.

\FloatBarrier

\subsection{Evaluation Protocol}
\label{app:evaluation}

Model quality is evaluated at intervals of 5,000 generator updates. Let
\(\mathcal T\) denote the set of evaluation checkpoints used for
best-checkpoint selection.

At each evaluation checkpoint, the generator is switched to evaluation
mode, and 50,000 latent vectors are independently sampled from
\[
Z\sim\mathcal N(0,I_{128}).
\]
The generated images are transformed from the generator output range
\([-1,1]\) to \([0,1]\) according to
\[
X_{\mathrm{eval}}
=
\frac{G(Z)+1}{2}.
\]
The additive uniform perturbation applied to real images during training
is not applied to generated images during metric evaluation.

Inception Score (IS) is computed from the 50,000 generated images.
Fr\'echet Inception Distance (FID) is computed between the generated
images and cached Inception-feature statistics of the CIFAR-10 training
set stored in \texttt{cifar10.train.npz}. The CIFAR-10 test set is not
used for training or for the FID reference statistics.

The evaluation code first imports the metric interface from
\path{pytorch-gan-metrics} and uses
\path{pytorch-image-generation-metrics} as a compatible fallback when
the former import is unavailable. The same metric implementation, number
of generated images, FID reference statistics, evaluation checkpoints,
and checkpoint-selection rule are used for all training strategies.

For each seed \(s\), the reported best IS and best FID are selected
independently over \(\mathcal T\):
\[
\mathrm{IS}^{\mathrm{best}}_s
=
\max_{t\in\mathcal T}
\mathrm{IS}_{s,t},
\qquad
\mathrm{FID}^{\mathrm{best}}_s
=
\min_{t\in\mathcal T}
\mathrm{FID}_{s,t}.
\]
Thus, the highest IS and lowest FID for a given seed may correspond to
different model checkpoints. Accordingly, the reported results represent
seed-wise best-checkpoint performance rather than performance at a single
common training checkpoint.

The aggregate values reported in the main paper are
\[
\overline{\mathrm{IS}}
=
\frac{1}{5}
\sum_{s=0}^{4}
\mathrm{IS}^{\mathrm{best}}_s,
\qquad
\overline{\mathrm{FID}}
=
\frac{1}{5}
\sum_{s=0}^{4}
\mathrm{FID}^{\mathrm{best}}_s.
\]

The reported uncertainties are sample standard deviations across the
five independently trained runs. For either metric \(M\), let
\(M_s\) denote the seed-wise best value and define
\[
\overline M
=
\frac{1}{5}
\sum_{s=0}^{4}M_s.
\]
The sample standard deviation is then
\[
s_M
=
\sqrt{
\frac{1}{4}
\sum_{s=0}^{4}
\left(
M_s-\overline M
\right)^2
}.
\]
The split-based standard deviation returned internally by an individual
IS computation is logged separately and is not used as the across-seed
standard deviation reported in the result tables.

\section{Computing Infrastructure}
\label{app:infrastructure}

The reported experiments were conducted using the hardware and software
configuration summarized in
Table~\ref{tab:computing_infrastructure}.

\begin{table}[!ht]
\centering
\caption{Computing infrastructure used for the experiments.}
\label{tab:computing_infrastructure}
\footnotesize
\setlength{\tabcolsep}{4pt}
\renewcommand{\arraystretch}{1.08}
\begin{tabular}{@{}p{0.38\columnwidth}p{0.54\columnwidth}@{}}
\toprule
Component & Configuration \\
\midrule
Available GPUs
& \(2\times\) NVIDIA RTX A5000, 24\,GB each \\

CPU
& Intel Xeon w7-2495X \\

System memory
& 125\,GiB \\

Operating system
& Ubuntu 22.04.5 LTS \\

Python
& 3.13.7 \\

PyTorch
& 2.9.1+\texttt{cu128} \\

PyTorch CUDA
& 12.8 \\

NVIDIA driver
& 575.51.03 \\

Driver CUDA
& 12.9 \\

cuDNN
& 9.10.2 \\
\bottomrule
\end{tabular}
\end{table}

IS and FID are evaluated using
\path{pytorch-gan-metrics} v0.5.4 and
\path{pytorch-image-generation-metrics} v0.6.1.

Each individual training run uses a single GPU. The two available GPUs are
used to execute independent runs in parallel.

\end{document}